\documentclass{article}

\usepackage[preprint,nonatbib]{neurips_2023}

\usepackage[dvipsnames,table,xcdraw]{xcolor}         

\usepackage{scalerel}
\usepackage[hyphens]{url}
\usepackage{minitoc}

\usepackage[numbers]{natbib}
\usepackage{subcaption}
\usepackage{bbm, dsfont} 
\usepackage{amssymb}
\usepackage{amsthm}
\usepackage{stmaryrd} 
\usepackage{times}
\usepackage{epsfig}
\usepackage{graphicx}
\usepackage{amsmath}
\usepackage{bbm, dsfont}
\usepackage{amssymb}
\usepackage{booktabs}
\usepackage{multirow}
\usepackage{pifont}
\newcommand{\cmark}{\textcolor{OliveGreen}{\ding{51}}}%
\newcommand{\xmark}{\textcolor{red}{\ding{55}}}%
\usepackage{algorithm}
\usepackage{algpseudocode}
\usepackage{comment}
\usepackage{caption}
\usepackage{tikz}
\usepackage{comment}
\usepackage{amsmath,amssymb} 
\usepackage{color}
\definecolor{Gray}{gray}{0.85}
\usepackage{array}
\usepackage{eqparbox}

\newcommand{\revision}[1]{\textcolor{black}{#1}}

\newtheorem{theorem}{Theorem}[section]

\newtheorem{lemma}[theorem]{Lemma}
\newtheorem{corollary}[theorem]{Corollary}

\usepackage[utf8]{inputenc} 
\usepackage[T1]{fontenc}    
\usepackage{url}            
\usepackage{booktabs}       
\usepackage{amsfonts}       
\usepackage{nicefrac}       
\usepackage{microtype}      

\title{REx: Data-Free Residual Quantization Error Expansion}

\author{%
  Edouard Yvinec$^{1,2}$ , Arnaud Dapogny$^2$ , Matthieu Cord$^1$ , Kevin Bailly$^{1,2}$\\
    Sorbonne Université$^1$, CNRS, ISIR, f-75005, 4 Place Jussieu 75005 Paris, France \\
  Datakalab$^2$, 114 boulevard Malesherbes, 75017 Paris, France \\
  \texttt{ey@datakalab.com} \\
}

\begin{document}

\maketitle

\begin{abstract}
Deep neural networks (DNNs) are ubiquitous in computer vision and natural language processing, but suffer from high inference cost. This problem can be addressed by quantization, which consists in converting floating point operations into a lower bit-width format. With the growing concerns on privacy rights, we focus our efforts on data-free methods. However, such techniques suffer from their lack of adaptability to the target devices, as a hardware typically only supports specific bit widths. Thus, to adapt to a variety of devices, a quantization method shall be flexible enough to find good accuracy \textit{v.s.} speed trade-offs for every bit width and target device. To achieve this, we propose REx, a quantization method that leverages residual error expansion, along with group sparsity. 
We show experimentally that REx enables better trade-offs (in terms of accuracy given any target bit-width) on both convnets and transformers for computer vision, as well as NLP models. In particular, when applied to large language models, we show that REx elegantly solves the outlier problem that hinders state-of-the-art quantization methods.
In addition, REx is backed off by strong theoretical guarantees on the preservation of the predictive function of the original model. Lastly, we show that REx is agnostic to the quantization operator and can be used in combination with previous quantization work.

\end{abstract}

\section{Introduction}
Deep neural networks (DNNs) achieve outstanding performance on several challenging computer vision tasks such as image classification \citep{he2016deep}, object detection \citep{liu2016ssd} and semantic segmentation \citep{chen2017rethinking} as well as natural language processing benchmarks such as text classification \citep{devlin2018bert}. However, their accuracy comes at a high computational inference cost which limits their deployment, moreso on edge devices when real-time treatment as well as energy consumption are a concern. This problem can be tackled \textit{via} DNN quantization, \textit{i.e.} by reducing the bit-width representation of the computations from floating point operations (FP) to e.g. int8 (8-bits integer representation), int4, int3 or even lower-bit representation such as ternary (where weights values are either $-1$, $0$ or $+1$) quantization. Because DNN inference principally relies on matrix multiplication, such quantization dramatically diminishes the number of bit-wise operations (as defined by \cite{krishnamoorthi2018quantizing}), thus limiting the DNN latency and energy consumption. However, DNN quantization usually comes at the expense of the network accuracy. As a consequence, DNN quantization is an active field of research \citep{courbariaux2016binarized,wu2018training,jacob2018quantization,achterhold2018variational,louizos2018relaxed,sheng2018quantization,choi2022s,zhong2022intraq} that aims at limiting this accuracy drop while reducing the number of bit-wise operations.

All the aforementioned methods are data-driven as they either involve training a network from scratch or fine-tune an already trained and quantized one. However, while such approaches usually allow lower quantization errors using low bit-wise representations, due to the growing concerns on privacy rights and data privacy, there is an ever increasing number of real-case scenarios (e.g. health and military services) where data may not be available for quantization purpose. \revision{Furthermore, the bloom of large langage models (LLMs) that are very expensive to train further motivates the use of \textit{post-hoc} data-free quantization methods.}
Motivated by these observations, several data-free quantization algorithms were published in the recent years \citep{nagel2019data,meller2019same,zhao2019improving,cai2020zeroq,zhang2021diversifying,squant2022}, which focus on the quantization operator, \textit{i.e.} the transformation which maps the floating point weights to their low-bit, fixed point, values. However, these approaches still struggle to offer an interesting alternative to data-driven techniques in terms of accuracy. 

Furthermore, when considering a specific target device for deployment, traditional quantization methods, usually focusing on the quantization operator, offer limited options: given a supported bit width (given by the device, as most hardware usually support only a few representation formats \citep{nivdiaA100}) they either achieve satisfactory accuracy or not. To address this concern, we wish to design a flexible quantization method, \textit{i.e.} one that can provide several accuracy \textit{vs.} speed trade-off points for each bit width. Drawing inspiration from wavelets-based methods for image compression \citep{rabbani2002jpeg2000,mallat2009theory}, we tackle this limitation by considering the successive residual quantization errors between the quantized and original model. Increasing the number of residuals in the expansion (\textit{i.e.} the expansion order) increases the fidelity to the original, non-quantized model at the expense of additional computations. In addition, we propose a group-sparse expansion which allows us to maintain the accuracy using significantly less bit operations. Hence, given a target device, our approach allows to find the best accuracy \textit{vs.} speed trade-off. 
Our contributions are thus four-fold:

\begin{itemize}
    \item \textbf{REx, a data-free quantization method that is both efficient and flexible.} REx leverages residual quantization, along with group-sparsity to enable finding suitable trade-offs depending on a target bit-width.
    \item \textbf{Theoretical guarantees} on both the exponential convergence of the quantized model towards the full-precision model and the maximum error with respect to the predictive function. This is of paramount importance in a data-free context, where we cannot easily measure the accuracy degradation.
    \item \textbf{Extensive empirical validation} we show through a thorough empirical validation that, as a standalone method, REx significantly outperforms every state-of-the-art data-free quantization technique, allowing to find better trade-offs on a variety of benchmarks involving ConvNet for classification, object detection or semantic segmentation as well as transformers on GLUE text classification.
    \item \textbf{A ready-to-use solution} that uses a single binary residual to handle outliers within the weight distributions, which is a well-known pitfall when attempting to quantize LLMs.
\end{itemize}

\section{Related Work}
\subsection{Quantization}
\revision{In this section, we review existing methods for DNN quantization, with an emphasis on approaches geared towards run-time acceleration.
The vast majority of DNN quantization techniques rely on data usage (Quantization Aware Training). Furthermore, methods such as \citep{wu2018training,jacob2018quantization,achterhold2018variational,louizos2018relaxed,sheng2018quantization,ullrich2017soft,zhou2016dorefa} rely on variants of straight through estimation to alleviate the rounding operation gradients.
Among these methods, \cite{oh2021automated} bears the most resemblance with the proposed REx method. It minimizes the residual error during training, using weight decay over the residue. The similarity with REx comes from the use of a second order expansion of the quantization errors. However, it discards the quantization error after training while we propose to keep the extra operations in order to ensure a high fidelity to the provided pre-trained model.}

\subsection{Data-Free Quantization}
Nagel \textit{et al. }\revision{\cite{nagel2019data} discuss the necessity to have data available so as to successfully design a quantization pipeline. They proposed a method that consists in balancing the weight ranges over the different layers of a model, using scale invariance properties that are specific to piece-wise affine (e.g. ReLU) activation functions, and relying on a traditional, naive quantization operator \citep{krishnamoorthi2018quantizing}. 
In the current state of data-free quantization research, we see two major trends: methods that focus on the rounding operator itself \citep{squant2022,yvinec2022spiq} and methods that generate synthetic data \citep{li2021mixmix,choi2022s,zhong2022intraq}.
With REx, we aim at enabling hardware flexibility for these methods by allowing to find better trade-offs in terms of accuracy and compression rate given a fixed bit-width.}

\subsection{Flexibility in Quantization}
\revision{In practice, the existing data-free quantization methods only offer a single possible quantized model given a supported bit-width. Nevertheless, most hardwares do not support a wide range of bit-width. For instance, Turing \citep{feng2021apnn} and Untether \citep{robinson_2022} architectures support int4 and int8 quantization while the Nvidia A100 \citep{nivdiaA100} supports int8, int4 and binary (int1) quantization. Conversely, REx circumvents this limitation by offering several trade-offs given a bit-width representation.}

\section{Methodology}

Let's consider $F$, a trained network with $L$ layers and trained weights $W_l$. Given a target integer representation in $b$ bits, e.g. int8 or int4, we consider a quantization operator $Q$. Formally, $Q$ maps $[\min\{W_l\}; \max\{W_l\}] \subset \mathbb{R}$ to the quantized interval $ [- 2^{b-1} ; 2^{b-1} -1] \cap \mathbb{Z}$. The most straightforward way to do so is to apply a scaling $s_{W_l}$ and round $\lfloor \cdot \rceil$ the scaled tensor, \textit{i.e.}:
\begin{equation}\label{eq:quantization_operator}
    Q(W_l) = \left\lfloor\frac{W_l}{\revision{s_{W_l}}}\right\rceil
\end{equation}
With $s_{W_l}$ the quantization scale for \revision{$W_l$ computed as in \cite{krishnamoorthi2018quantizing}, without loss of generality.} Following the standard formulation \citep{gholami2021survey}, a quantization operator $Q$, comes with a de-quantization operator $Q^{-1}$. For the simple quantization operator $Q$ in Equation \eqref{eq:quantization_operator}, a natural choice is $Q^{-1}(Q(W_l)) = s_{W_l} \times Q(W_l)$. Note that, despite the notation, $Q^{-1}$ is not a true inverse , as by definition of the quantized space, there is some loss of information. This loss, called the quantization error, is defined as: $W_l - Q^{-1}(Q(W_l))$.
In data-free quantization, we want to minimize this error in order to achieve the highest possible fidelity to the original model.
In the following section, we describe how we can efficiently reduce the quantization error for a fixed target bit-width $b$.

\begin{figure*}[!t]
    \centering
    \includegraphics[width = 0.58\linewidth]{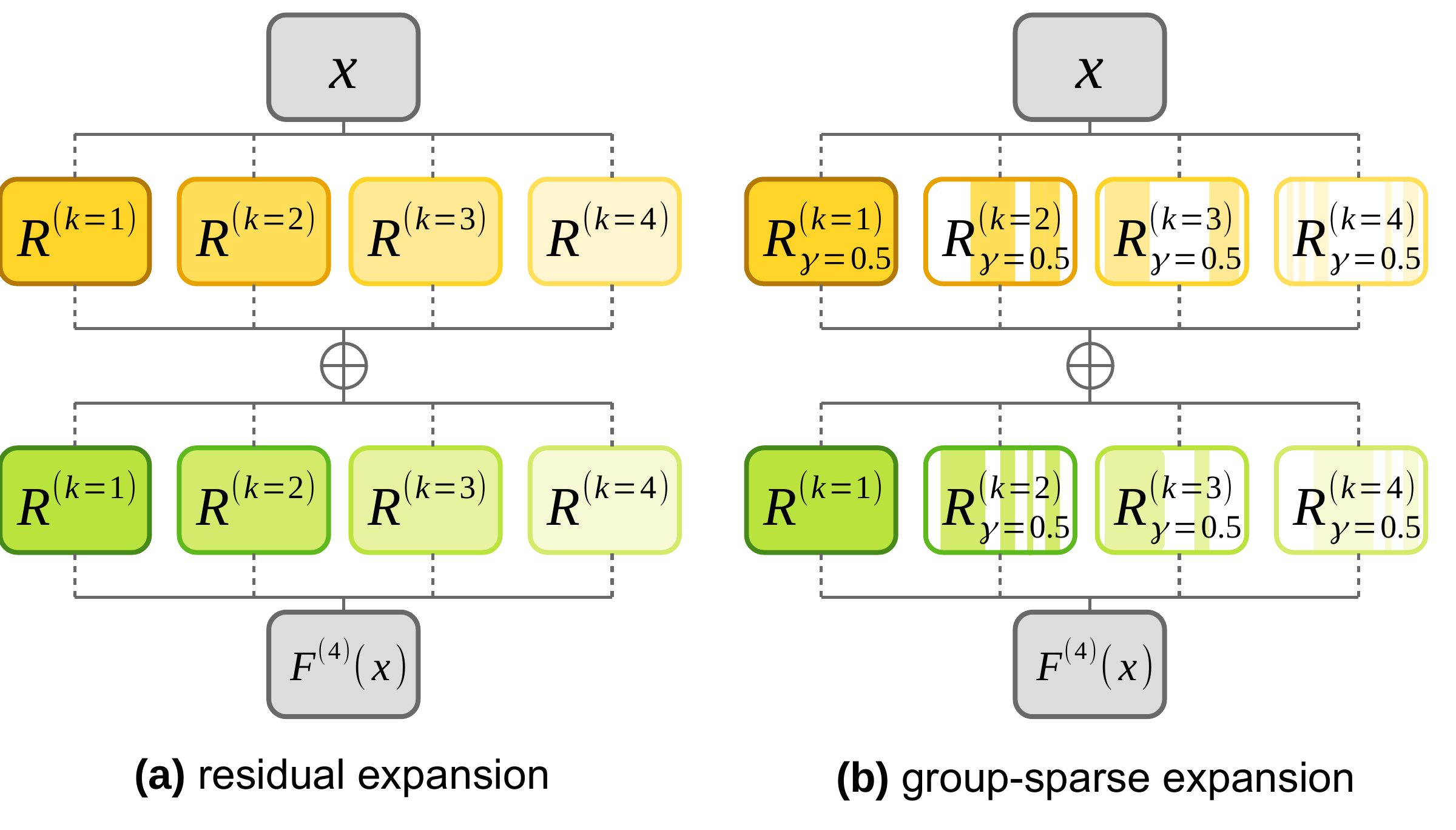}
    \caption{Illustration of the proposed method for a two-layers neural network. \textbf{(a)} residual expansion at order $4$: the intensity of the colormap indicates the magnitude of the residual error. \textbf{(b)} group-sparse expansion for orders $k \geq 1$ ($\gamma = 50\%$ sparsity).}
    \label{fig:DRE_explain}
\end{figure*}

\subsection{Residual Expansion}
We propose to quantize the residual errors introduced by the quantization process. Although the proposed method can be applied to any tensor, let's consider a weight tensor $W$. In the full-precision space ($\mathbb{R}$), its first approximation is $R^1 = Q^{-1}(Q(W))$. To reduce the quantization error, we define $R^2$ as the quantized residual error
\begin{equation}\label{eq:residu_2_definition}
    R^{2} = Q^{-1}\left(Q\left(W - R^1 \right)\right)
\end{equation}
Consequently, during the quantized inference, we compute $R^1X + R^2X\approx WX$ which provides a finer approximation than the simple evaluation $R^1X$. The process can be generalized to any expansion order $K$, leading to the following: 
\begin{equation}\label{eq:residu_definition}
    R^{K} = Q^{-1}\left(Q\left(W - \sum_{k=1}^{K-1} R^k \right)\right)
\end{equation}
The resulting expanded layer is illustrated in Figure \ref{fig:DRE_explain} (a) in the case $K=4$.
Intuitively, an expansion $(R^1,...,R^K)$ provides the approximation $\sum_{k=1}^K R^k$ of $W$ and this approximation converges exponentially fast to the original full-precision weights with respect to $K$. As the support of the quantization error space is smaller than one quantization step, the error decreases by a factor larger than $2^b$ with each expansion term (more details in Appendix \ref{appendix:convergence}). Furthermore, as the quantization error decreases, it is expected that the prediction of the quantized model would achieve a closer match to the original one. This is especially important in the context of data-free quantization as not only do we not have the option to perform fine-tuning to recover accuracy, but also we cannot evaluate the degradation of the model on a calibration/validation set. Nonetheless, we can estimate an upper bound on the maximum error $\epsilon_{\max}$ introduced by quantization on the predictions as 
\begin{equation}\label{eq:main_upper_bound}
    \epsilon_{\max} \leq U = \prod_{l=1}^L \left(\sum_{i=1}^l \left(\frac{1}{2^{b-1}-1}\right)^{K-1}\frac{s_{R^{i}}}{2} + 1 \right) - 1
\end{equation}
\revision{where $s_{R^{i}}$ is the scaling factor from equation \ref{eq:quantization_operator} applied to each residue.}
The detailed derivations are provided in Appendix \ref{appendix:upperbound}. This implies that, in practice and regardless on the quantization operator, a network can be quantized with high fidelity with only a few expansion orders to fit a given bit-width. \revision{Furthermore, this process can also be applied to the activations.}

\subsection{\revision{Input Expansion}}
\revision{Quantizing the weights of a DNN with the aforementioned method already leads to significant memory footprint reduction. However, to significantly decrease the inference runtime, the inputs and activations of each layer also have to be quantized so that each the computations can be processed in the quantized bit-width. For that matter, let $I$ be the input tensor of a layer $l$. We define the expansion of $I$ in quantized residual errors similarly to the weights expansion. Using the generic quantization operator $Q$, we get $I^{(1)} = Q(I)$ and define the $K^\text{th}$ order of quantization as 
\begin{equation}\label{eq:input_expansion}
    I^{(K)} = Q\left(I - \sum_{k=1}^{K-1} Q^{-1}(I^{(k)})\right)
\end{equation}
In order to efficiently exploit the resulting expansions, we propose to bound the accumulated order of the weights and inputs. In other words, if we note $k_1$ the expansion order of a residue from the inputs and $k_2$ a residue from the weights, then we only perform the computations for orders such that $k_1 + k_2 < K$ (the rest being negligible in comparison). As a result, the quantized layer $l$ computes:
\begin{equation}\label{eq:baseline_expansion}
    f: I \mapsto \sum^{k_1 + k_2 \leq K+1}_{
        k_1, k_2 \in \{1,\dots,K\}^2} Q^{-1}\left(I^{(k_1)}\otimes R^{(k_2)}\right)
\end{equation}
where $\otimes$ is the base operation of the layer, e.g. a convolution for a convolutional layer or a matrix multiplication for a fully-connected layer.
Similarly to the weights, the error between the full-precision inputs and the proposed expansion of the inputs converges exponentially fast to $0$ with respect to the order $K$ of the expansion. 
However, with formulations from equations \eqref{eq:residu_definition} and \eqref{eq:input_expansion}, the overhead computations induced by the expansion is non-negligible.
In the following section, we provide a solution to tackle this issue. }

\subsection{Sparse Expansion}
The residual expansion as defined in equation \ref{eq:residu_definition} is based upon the assumption that the quantization error is equally important for every neuron. Thus, we propose to reduce the overhead cost by only expanding the most important neurons. However, in data-free compression we do not have access to activations or gradients: hence, we measure the relative importance of a neuron in a layer by the norm of its weights \citep{molchanov2016pruning}.
The resulting expanded layer is illustrated in Figure \ref{fig:DRE_explain} (b).
Given a target budget $\gamma$ (in \%) of overhead computations, we only expand the $\frac{\gamma}{K-1}$\% most important neurons.
The sparse residue is defined as:
\begin{equation}\label{eq:sparsity_definition}
    \left(R^{(k)}_\gamma\right)_i = (R^{(k)})_i \cdot \mathbbm{1}^{(k)}_\gamma
\end{equation}
where $\mathbbm{1}^{(k)}_\gamma$ indicates the indices of the most important neurons.
Similarly to what precedes, each expansion order is derived sequentially from previous orders and we can bound the quantization error for the sparse expansion (see Appendix \ref{appendix:convergence}).
\revision{The method for computing the weights of the expanded model is summarized in Algorithm \ref{alg:expansion}. 
\begin{algorithm}
\caption{Expansion Algorithm}\label{alg:expansion}
\begin{algorithmic}
\Require trained DNN $f$ with $L$ layers, hyper-parameters : $K$ and $\gamma$, operator $Q$
\State initialize $\gamma^l$ and initialize $f^{(K)}$ as a clone of $f$ with $K$ per-layer kernels
\For{$l \in\{1,\dots,L\}$}
    \State $W\leftarrow$ base kernel of layer $l$ in $f$
    \State $W_{\text{acc}} \leftarrow 0$ accumulated quantization error
    \For{$k \in\{1,\dots,K\}$}
        \State $R^{(k)}_{\gamma^l}\leftarrow Q(W-W_{\text{acc}}) \mathbbm{1}^{(k)}_\gamma$ \Comment{equation \ref{eq:sparsity_definition}}
        \State set $k^{\text{th}}$ kernel of layer $l$ of $f^{(K)}$ with $R^{(k)}_{\gamma^l}$
        \State $W_{\text{acc}} \leftarrow W_{\text{acc}} + Q^{-1}(R^{(k)}_{\gamma^l})$
    \EndFor
\EndFor
\State return $f^{(K)}$
\end{algorithmic}
\end{algorithm}
}

Also note that in the sparse expansion, we allow higher expansion orders to re-consider neurons that were previously considered unimportant. Consequently, on top of improving the exponential convergence as well as lowering the upper bound on the maximum error with respect to the overhead computations, this method systematically outperforms the standard residual expansion in practice. Proof of this result can be found in Appendix \ref{appendix:pyramid}.
The budget $\gamma$ of overhead computations can be set so as not to introduce computational overhead, depending on the bit-width $b$. For example, let's consider a device supporting only 8 and 1 bit (binary) quantization. If we want to achieve the same latency as 8 bit quantization using only 1bit quantization we will have a budget lower than $700\%$ overhead w.r.t. a naive 1 bit quantization. Consequently, for full expansions, we get $\gamma\ \leq \frac{8}{1} - 1 = 700\%$. This budget is then split across layers using a simple linear repartition. This strategy gives more emphasis to the layers closest to the prediction head which also correspond to the largest layers, and empirically provides the best results \cite{yvinec2021red}.
As a result, given a number bit operations (BOPS), the expanded model can better fit the inference device while preserving the full-precision accuracy. Furthermore, all the added computations are performed in parallel which reduces their cost in practice. It allows better trade-offs in terms of accuracy and quantization compression rate, as will be shown in the upcoming experiments.

\section{Quantization Experiments}

In the following sections, we first go through the implementation requirements and efficient strategies to fully leverage the proposed expansions. Second, we perform a comparison of each expansion methods in order to show the flexibility of REx with respect to the bit-width. Third, we compare REx to other quantization schemes under the constraint of equal bit operations. Finally, we validate for each expansion their respective upper bound on the maximum error with respect to the original predictions. 

\subsection{Implementation Details and Benchmarks}

We ran our tests on 6 different backbones, including ConvNets and transformers and 5 tasks from both computer vision and natural language processing. We used ImageNet \citep{imagenet_cvpr09}, Pascal VOC 2012 \citep{pascal-voc-2012}, CityScapes dataset \citep{cordts2016cityscapes} and GLUE \citep{wang-etal-2018-glue} and common sense reasoning benchmarks (details in Appendix \ref{appendix:implem}). 

Unless stated otherwise, we apply symmetric, static, per-channel quantization as defined in \cite{gholami2021survey} and perform batch-normalization folding prior to any processing using the optimal method from \cite{yvinec2022fold}. In order to leverage the existing efficient implementations of the convolutional layers and fully-connected layers in CUDA, we propose to implement the expanded layer using a single kernel rather than $K$ kernels. This is achieved by concatenating the kernels along the output dimension. Consequently, the challenge of efficiently splitting the computations to fully leverage the target device computational power is left to the inference engine. In practice, this results in both better performance and less work in order to adapt the method to existing engines such as OpenVino \citep{openvino} and TensorRT \citep{tensorrt}. We detail the implementation and overhead of the addition of the residual computations in Appendix \ref{sec:appendix_sum}. In the following section, we demonstrate the ability of REx to find good accuracy \textit{vs} speed trade-offs.

\subsection{Flexible Quantization}
\begin{figure}[!t]
    \centering
    \includegraphics[width = 0.9\linewidth]{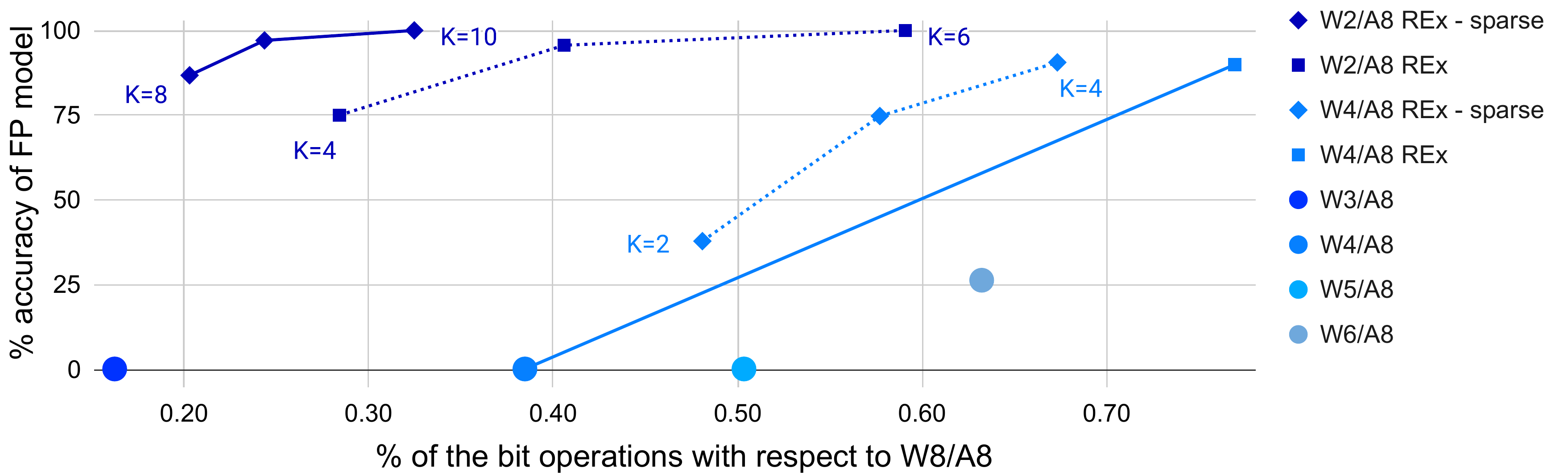}
    \caption{Accuracy \textit{vs.} inference time, for EfficientNet B0. The higher (accuracy) and the further to the left (inference cost) the better. The circles show the baseline results with W3/A8, W4/A8, W5/A8 and W6/A8 quantization. The dashed lines show the trade-offs performance of REx in W4/A8 and ternary quantization (W2/A8). Finally, the plain lines show REx (with sparsity at $10\%$) also in W4/A4 and ternary quantization. The numbers in the symbols stands for the expansion order. REx, and \textit{a fortiori} the sparse version, enables better trade-offs.}
    \label{fig:ablation}
\end{figure}

Figure \ref{fig:ablation} shows different trade-offs enabled by REx on different bit-widths for an EfficientNet-B0 on ImageNet. First, the baseline quantization with the baseline quantization operator from \cite{krishnamoorthi2018quantizing} (as depicted by the circles of different colors, one for each bit width) offers no trade-off possibility given a specific bit-width and usually performs poorly below int8 quantization (e.g. barely reaching $20.29\%$ top1 accuracy in W6/A8 quantization). REx, however, in the same setup, offers several trade-offs for each specific bit-width (e.g. int4 and ternary on Figure \ref{fig:ablation}) and supporting hardware. Furthermore, the sparse expansion enables finding more potential trade-offs (by varying the budget and expansion order) for every bit-width. Those trade-offs are generally more interesting than comparable ones obtained using the baseline method, which empirically confirms the theoretical results (Appendix \ref{appendix:pyramid}).

Furthermore, Figure \ref{fig:ablation} shows that using higher order, sparse residues allows to find even better trade-offs, as, in this case, e.g. in W2/A8 we reach full-precision accuracy at order $10$ with 10\% sparse residues. This shows that the process converges fast with respect to the sparsity rates. All in all, these results demonstrate the flexibility of REx to find good accuracy \textit{v.s.} speed trade-offs, given a budget of total bit operations (BOPs) to fit. In the following section, we evaluate the ability of REx to outperform existing quantization methods in terms of equal bops.

\subsection{Main Results}\label{compsota}
\subsubsection{Experiments on Computer Vision Models}

In order to highlight the benefits of residual quantization errors expansions as a stand alone improvement upon existing methods with equal BOPs, we compare REx using the naive quantization operator from \citep{krishnamoorthi2018quantizing} on a variety of reference benchmarks. First, in Table \ref{tab:compar_sota_equalbits}, we report the performance on three different computer vision networks between state-of-the-art methods in W6/A6 quantization and REx using a sparse expansion at order $K=2$ using $50\%$ of a 4 bit representation in order to get a similar total number of bit operations (150\% of 4 bits $\approx$ 6 bits). For all networks, REx significantly outperforms recent state-of-the-art data-free quantization methods at equal BOPs. Furthermore, we confirm these results on object detection and image segmentation as shown in Figure \ref{fig:ssd}. We can observe that REx can maintain the full precision accuracy while dividing by $3.23$ the number of bit operations required to run an inference.

\begin{table}[!t]
\caption{Comparison at equal BOPs with existing methods in W6/A6 and REx with W4/A6 +50\% of one 4 bit residue.}
\label{tab:compar_sota_equalbits}
\centering
\setlength\tabcolsep{3pt}
        \begin{tabular}{c|c|c|c|c}
         \hline
         DNN & method & year & bits & Accuracy \\
         \hline
         \multirow{8}{*}{ResNet 50} & \multicolumn{3}{c|}{full-precision} & 76.15 \\
         \cline{2-5}
         & DFQ \cite{nagel2019data} & ICCV'19 & \revision{W6/A6} & 71.36 \\
         & ZeroQ \cite{cai2020zeroq} & CVPR'20 & \revision{W6/A6} & 72.93\\
         & DSG \cite{zhang2021diversifying} & CVPR'21 & \revision{W6/A6} & 74.07 \\
         & GDFQ \cite{xu2020generative} & ECCV'20 & \revision{W6/A6} & 74.59 \\
         & SQuant \cite{squant2022} & ICLR'22 & \revision{W6/A6} & 75.95 \\
         & SPIQ \cite{yvinec2022spiq} & WACV'23 & \revision{W6/A6} & 75.98 \\
         & REx & - & 150\% $\times$ \revision{W4/A6} & \textbf{76.01} \\
         \hline
         \multirow{5}{*}{MobNet v2} & \multicolumn{3}{c|}{full-precision} & 71.80 \\
         \cline{2-5}
        & DFQ \cite{nagel2019data} & ICCV'19 & \revision{W6/A6} & 45.84 \\
        & SQuant \cite{squant2022} & ICLR'22 & \revision{W6/A6} & 61.87 \\
        & SPIQ \cite{yvinec2022spiq} & WACV'23 & \revision{W6/A6} & 63.24 \\
        & REx & - & 150\% $\times$ \revision{W4/A6} & \textbf{64.20} \\
         \hline
         \multirow{4}{*}{EffNet B0} & \multicolumn{3}{c|}{full-precision} & 77.10 \\
         \cline{2-5}
        & DFQ \cite{nagel2019data} & ICCV'19 & \revision{W6/A6} & 43.08 \\
        & SPIQ \cite{yvinec2022spiq} & ICLR'22 & \revision{W6/A6} & 54.51 \\
        & REx & - & 150\% $\times$ \revision{W4/A6} & \textbf{57.63} \\
         \hline
        \end{tabular}
\end{table}

\begin{figure}[!t]
\centering
\includegraphics[width = \linewidth]{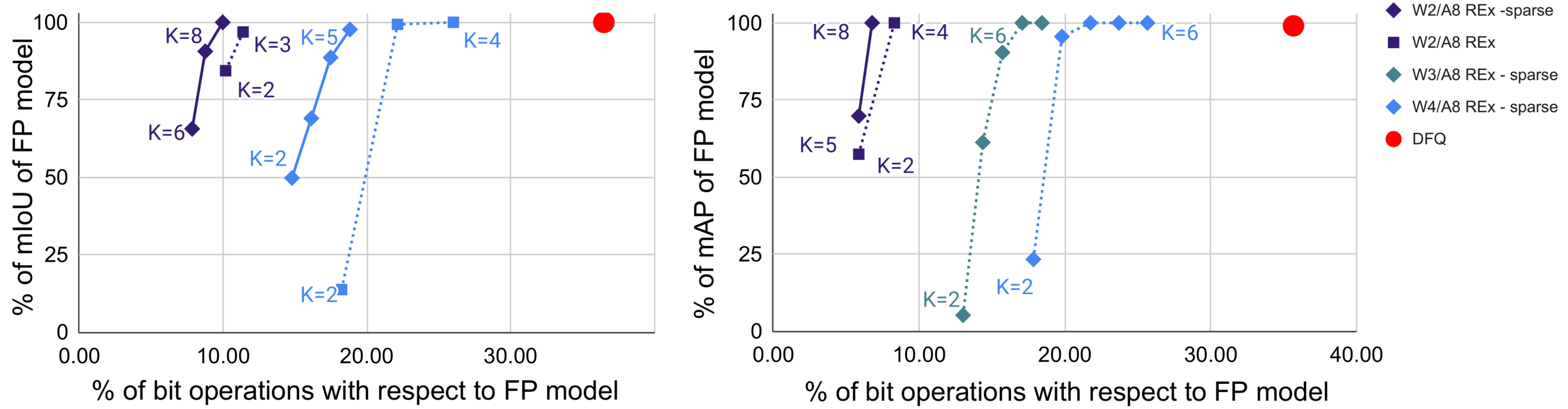}
\caption[Two numerical solutions]{(left) Mean intersection over union (mIoU) of a Deeplab V3+ with MobileNet V2 backbone on CityScapes for semantic segmentation. (right) Mean average precision (mAP) of a SSD with MobileNet V2 backbone on Pascal VOC for object detection. We add the performance of a data-free quantization solution, DFQ \citep{nagel2019data} for comparison. }
\label{fig:ssd}
\label{fig:deeplab}
\end{figure}

\subsubsection{Experiments on NLP}
In Table \ref{tab:comparison_sota_bert}, we perform a similar experiment on NLP using Bert \citep{devlin2018bert}. We can observe the generalization of our results from ConvNets to Transformers, REx can find better accuracy per bits trade-offs as compared to four references including non-uniform quantization \citep{miyashita2016convolutional}. Bert is a pre-trained model with 86 million parameters which is now considered a medium sized model. Both the full-precision and quantized models can fit on a single middle range GPU. However, recent state-of-the-art models, such as OPT \cite{zhang2022opt}, are so large that they need multiple gpus just to be loaded on memory. These models ought to be compressed for sustainable usage.
In the following section, we generalize the performance of REx to extreme model sizes.

\begin{table*}[!t]
\caption{GLUE task quantized in W4/A8. We consider the BERT transformer architecture \citep{devlin2018bert} and provide the original performance from the article (original) of BERT on GLUE as well as our reproduced results (reproduced). REx is applied to the weights with 3 bits + 33\% sparse expansion.}
\label{tab:comparison_sota_bert}
\centering
\setlength\tabcolsep{1pt}
\begin{subtable}{.35\textwidth}
\hfill
\begin{tabular}{|c|c|c|}
\hline
task & original & reproduced \\
\hline
CoLA & 49.23 & 47.90 \\
SST-2 & 91.97 & 92.32 \\
MRPC & 89.47/85.29 & 89.32/85.41 \\
STS-B & 83.95/83.70 & 84.01/83.87 \\
QQP & 88.40/84.31 & 90.77/84.65 \\
MNLI & 80.61/81.08 & 80.54/80.71 \\
QNLI & 87.46 & 91.47 \\
RTE & 61.73 & 61.82 \\
WNLI & 45.07 & 43.76 \\
\hline
\end{tabular}
\end{subtable}
\begin{subtable}{.63\textwidth}
\begin{tabular}{|c|c|c|c|c|c|}
\hline
uniform \cite{krishnamoorthi2018quantizing} & log \cite{miyashita2016convolutional} & SQuant \cite{squant2022} & SPIQ \cite{yvinec2022spiq} & REx \\
\hline
45.60 & 45.67 & \underline{46.88} & 46.23 & \textbf{47.02} \\
\underline{91.81} & 91.53 & 91.09 & 91.01 & \textbf{91.88} \\
88.24/84.49 & 86.54/82.69 & \textbf{88.78}/\textbf{85.24} & 88.78/85.06 &  \underline{88.71/85.12}\\
\underline{83.89}/\underline{83.85} & {84.01}/83.81 & 83.80/83.65 & 83.49/83.47 & \textbf{83.92}/\textbf{83.85} \\
89.56/83.65 & 90.30/84.04 & \underline{90.34}/\underline{84.32} & 90.30/84.21 & \textbf{90.50}/\textbf{84.35} \\
\underline{78.96}/79.13 & 78.96/79.71 & 78.35/79.56 & 78.52/\underline{79.86} & \textbf{79.03}/\textbf{79.96} \\
89.36 & 89.52 & \textbf{90.08} & \underline{89.64} & \textbf{90.08} \\
\underline{60.96} & 60.46 & 60.21 & 60.21 & \textbf{61.20} \\
39.06 & 42.19 & \underline{42.56} & 42.12 & \textbf{42.63} \\
\hline
\end{tabular}
\end{subtable}
\end{table*}

\subsubsection{Application to Handling Outliers in LLMs}
\revision{A known pitfall \cite{dettmers2022llm} for quantization on LLMs, comes from the presence of extreme outliers among their weight values. These outliers stretch out the weight values range and increase the scaling factor in Equation \ref{eq:quantization_operator}, which, in turn, causes smaller weights to be rounded abruptly to zero. Worse, as suggested in \cite{dettmers2022llm}, this phenomenon seems to occur more as the model size increases and might appear as a major problem for future work in large DNN quantization. In order to overcome this challenge, we adapt REx to only quantize the outliers in a residue using binary values (W1/A16) while the remaining weights are quantized in int4 (W4/A16). As a result, the overhead from REx is limited to a binary expansion with over 99.8\% sparsity. As listed in Table \ref{tab:comparison_sota_llm}, our evaluation on common sense reasoning tasks demonstrates that REx provides a significant improvement over other quantization operators at virtually no cost. Hence, REx appears as a ready-to-use solution to the outlier problem for quantization of LLMs.}

\begin{table}[!t]
\caption{Evaluation on Common sense reasoning benchmarks for OPT-13B \cite{zhang2022opt} LLM quantized in W4/A16. For each quantization operator DFQ \cite{nagel2019data}, SQuant \cite{squant2022} and PowerQuant \cite{yvinec2023power}, we share performance with and without REx (noted with check marks). We also provide the original full-precision (FP) performance.}
\label{tab:comparison_sota_llm}
\centering
\begin{tabular}{cclccllllcc}
\cline{2-2} \cline{4-5} \cline{7-8} \cline{10-11}
\multicolumn{1}{c|}{} & \multicolumn{1}{c|}{FP} & \multicolumn{1}{c|}{} & \multicolumn{2}{c|}{DFQ \cite{nagel2019data}} & \multicolumn{1}{c|}{} & \multicolumn{2}{c|}{SQuant \cite{squant2022}} & \multicolumn{1}{c|}{} & \multicolumn{2}{c|}{PowerQuant \cite{yvinec2023power}} \\ \cline{1-2} \cline{4-5} \cline{7-8} \cline{10-11} 
\multicolumn{1}{|c|}{\textbf{Use REx}} & \multicolumn{1}{c|}{-} & \multicolumn{1}{c|}{} & \multicolumn{1}{c|}{\xmark} & \multicolumn{1}{c|}{\cmark} & \multicolumn{1}{c|}{} & \multicolumn{1}{c|}{\xmark} & \multicolumn{1}{c|}{\cmark} & \multicolumn{1}{c|}{} & \multicolumn{1}{c|}{\xmark} & \multicolumn{1}{c|}{\cmark} \\ \cline{1-2} \cline{4-5} \cline{7-8} \cline{10-11} 
\multicolumn{1}{|c|}{HellaSwag} & \multicolumn{1}{c|}{52.43} & \multicolumn{1}{c|}{} & \multicolumn{1}{c|}{49.25} & \multicolumn{1}{c|}{\textbf{50.14}} & \multicolumn{1}{c|}{} & \multicolumn{1}{c|}{49.23} & \multicolumn{1}{c|}{\textbf{50.21}} & \multicolumn{1}{c|}{} & \multicolumn{1}{c|}{\textbf{51.29}} & \multicolumn{1}{c|}{50.98} \\ \cline{1-2} \cline{4-5} \cline{7-8} \cline{10-11} 
\multicolumn{1}{|c|}{OpenBookQA} & \multicolumn{1}{c|}{27.20} & \multicolumn{1}{c|}{} & \multicolumn{1}{c|}{\textbf{25.80}} & \multicolumn{1}{c|}{25.40} & \multicolumn{1}{c|}{} & \multicolumn{1}{c|}{25.40} & \multicolumn{1}{c|}{\textbf{26.20}} & \multicolumn{1}{c|}{} & \multicolumn{1}{c|}{25.80} & \multicolumn{1}{c|}{\textbf{27.80}} \\ \cline{1-2} \cline{4-5} \cline{7-8} \cline{10-11} 
\multicolumn{1}{|c|}{ARC-E} & \multicolumn{1}{c|}{61.91} & \multicolumn{1}{c|}{} & \multicolumn{1}{c|}{59.93} & \multicolumn{1}{c|}{\textbf{61.91}} & \multicolumn{1}{c|}{} & \multicolumn{1}{c|}{59.97} & \multicolumn{1}{c|}{\textbf{61.95}} & \multicolumn{1}{c|}{} & \multicolumn{1}{c|}{\textbf{60.82}} & \multicolumn{1}{c|}{60.52} \\ \cline{1-2} \cline{4-5} \cline{7-8} \cline{10-11} 
\multicolumn{1}{|c|}{ARC-C} & \multicolumn{1}{c|}{32.94} & \multicolumn{1}{c|}{} & \multicolumn{1}{c|}{30.2} & \multicolumn{1}{c|}{\textbf{32.42}} & \multicolumn{1}{c|}{} & \multicolumn{1}{c|}{30.12} & \multicolumn{1}{c|}{\textbf{32.34}} & \multicolumn{1}{c|}{} & \multicolumn{1}{c|}{31.57} & \multicolumn{1}{c|}{\textbf{32.94}} \\ \cline{1-2} \cline{4-5} \cline{7-8} \cline{10-11} 
\multicolumn{1}{|c|}{Winogrande} & \multicolumn{1}{c|}{65.04} & \multicolumn{1}{c|}{} & \multicolumn{1}{c|}{64.56} & \multicolumn{1}{c|}{\textbf{64.72}} & \multicolumn{1}{c|}{} & \multicolumn{1}{c|}{64.48} & \multicolumn{1}{c|}{\textbf{64.88}} & \multicolumn{1}{c|}{} & \multicolumn{1}{c|}{64.88} & \multicolumn{1}{c|}{\textbf{65.04}} \\ \cline{1-2} \cline{4-5} \cline{7-8} \cline{10-11} 
\multicolumn{1}{|c|}{PiQA} & \multicolumn{1}{c|}{76.88} & \multicolumn{1}{c|}{} & \multicolumn{1}{c|}{75.84} & \multicolumn{1}{c|}{\textbf{76.17}} & \multicolumn{1}{c|}{} & \multicolumn{1}{c|}{75.84} & \multicolumn{1}{c|}{\textbf{76.30}} & \multicolumn{1}{c|}{} & \multicolumn{1}{c|}{75.90} & \multicolumn{1}{c|}{\textbf{76.93}} \\ \cline{1-2} \cline{4-5} \cline{7-8} \cline{10-11} 
\multicolumn{1}{|c|}{BoolQ} & \multicolumn{1}{c|}{65.90} & \multicolumn{1}{c|}{} & \multicolumn{1}{c|}{54.71} & \multicolumn{1}{c|}{\textbf{65.54}} & \multicolumn{1}{c|}{} & \multicolumn{1}{c|}{54.28} & \multicolumn{1}{c|}{\textbf{65.38}} & \multicolumn{1}{c|}{} & \multicolumn{1}{c|}{\textbf{70.43}} & \multicolumn{1}{c|}{69.45} \\ \cline{1-2} \cline{4-5} \cline{7-8} \cline{10-11} 
\multicolumn{1}{l}{} & \multicolumn{1}{l}{} &  & \multicolumn{1}{l}{} & \multicolumn{1}{l}{} &  &  &  &  & \multicolumn{1}{l}{} & \multicolumn{1}{l}{} \\ \cline{1-2} \cline{4-5} \cline{7-8} \cline{10-11} 
\multicolumn{1}{|c|}{Average Score} & \multicolumn{1}{c|}{54.61} & \multicolumn{1}{c|}{} & \multicolumn{1}{c|}{51.47} & \multicolumn{1}{c|}{\textbf{53.76}} & \multicolumn{1}{c|}{} & \multicolumn{1}{c|}{51.33} & \multicolumn{1}{c|}{\textbf{53.91}} & \multicolumn{1}{c|}{} & \multicolumn{1}{c|}{54.38} & \multicolumn{1}{c|}{\textbf{54.81}} \\ \cline{1-2} \cline{4-5} \cline{7-8} \cline{10-11} 
\end{tabular}
\end{table}

\subsection{Empirical Validation of the Theoretical Bounds}

\begin{table}[!t]
\caption{Upper bound $U$ (see theorem \ref{thm:dre_upperbound} and \ref{thm:slim_upperbound}) over the maximum error as compared to the corresponding empirical measurement $U_{\text{empirical}}$ of that error for a VGG 16 \citep{simonyan2014very} trained on ImageNet. The closer the upper bound $U$ to the value $U_{\text{empirical}}$ the better.}
\label{tab:upper_bound}
\centering
\setlength\tabcolsep{2pt}
\begin{tabular}{c|c|c|c|c}
\hline
weights bit-width & expansion order $K$ & sparsity & $U$ & $U_{\text{empirical}}$\\
\hline
8 & 1 & \xmark & 0.12 & 0.05\\
8 & 4 & \xmark & 1.99 $\times 10^{-7}$ & 1.78 $\times 10^{-7}$ \\
8 & 2 & 50\% & 0.06 & 0.05 \\
8 & 4 & 50\% & 1.17 $\times 10^{-7}$ & 0.65 $\times 10^{-7}$ \\
\hline
\end{tabular}
\end{table}

Having shown the interest of REx for quantizing various architectures for computer vision and NLP tasks, we now empirically confirm its mathematical guarantees. In Table \ref{tab:upper_bound}, we validate the proposed upper bound $U$ in Equation \ref{eq:main_upper_bound} on the maximum error on the predictions on a VGG-16 \citep{simonyan2014very} trained on ImageNet. 
The tightness of the provided theoretical results can be estimated from the gap between our estimation and the empirical maximum error $U_{\text{empirical}}$ from quantization on the predictions, which is measured as the infinite norm between the full-precision and quantized logits. We observe that a naïve 8-bits quantization (\textit{i.e.} no expansion) leads to an upper bound $U = 0.12$, while we observe $U_{\text{empirical}}=0.05$. The norms of the logits is equal to $0.3423$. Therefore, the proposed upper bound is relatively tight and significantly lower than the logits magnitude: in such a case, due to overconfidence, the error shouldn't affect the classification. The proposed upper bound is even tighter for larger values of $K$, and becomes lower and lower (for both the theoretical and corresponding empirical maximum errors) when introducing sparsity. This further demonstrates the good properties of the proposed expansion approximation in REx in addition to the relevance of its theoretical guarantees, which are critical in data-free quantization.

\subsection{Flexibility with respect to the Quantization Operator}
\begin{table}[!t]
    \centering
    \setlength\tabcolsep{2pt}
    \caption{We report the different trade-offs achieved with REx expanding over different proposed quantization operators in \revision{W4/A4} as compared to their performance in \revision{W8/A8}, \revision{on a MobileNet V2}.}
    \begin{tabular}{c|c|c|c|c|c|c}
\hline
    method & W4/A4 & \revision{$\text{W4}_{\text{+ 25\%}}$}/A4 & \revision{$\text{W4}_{\text{+ 50\%}}$}/A4 & \revision{$\text{W4}_{\text{+ 75\%}}$}/A4 & W6/A6 &  W8/A8 \\
\hline
\hline
        naive \cite{krishnamoorthi2018quantizing} & 0.1 & 53.11 & 64.20 & \textbf{71.61} & 51.47 & 70.92 \\
        SQuant \cite{squant2022} & 4.23 & 58.64 & 67.43 & \textbf{71.74} & 60.19 & 71.68 \\
        SPIQ \cite{yvinec2022spiq} & 5.81 & 59.37 & 68.82 & \textbf{71.79} & 63.24 & \textbf{71.79} \\
        AdaRound \cite{nagel2020up} & 56.17 & 61.30 & 69.80 & \textbf{71.77} & 68.71 & \textbf{71.75} \\
        BrecQ \cite{li2021brecq} & 66.57 & 70.94 & 71.28 & \textbf{71.76} & 70.45 & \textbf{71.76} \\
\hline
    \end{tabular}
    \label{tab:flexibility_operator}
\end{table}

Most recent approaches for data-free quantization focus on designing better quantization operators. Interestingly, as we already hinted on large language models, our approach is agnostic to the choice of the quantization operator and can thus be combined with these approaches without bells and whistles. In Table \ref{tab:flexibility_operator}, we report the possible trade-offs achievable with REx combined with recent approaches focusing on the quantization operator \revision{on MobileNet V2}. 
\revision{The different trade-offs are sorted in ascending order in terms of added overhead operations, e.g. $\text{W4}_{\text{+ 25\%}}$ leads to less operations than $\text{W4}_{\text{+ 50\%}}$.}
First, when used with SQuant \citep{squant2022}, REx achieves full-precision accuracy in W4/A4 \revision{with only $75\%$ overhead}, even outperforming W8/A8 quantization. SPIQ \citep{yvinec2022spiq}, can also be adapted with REx in order to achieve good accuracy using only 4 bits representation as it benefits from finer weight quantization. This explains the slightly higher accuracies than SQuant using 25\% and 50\% sparsity. Finally, with AdaRound \citep{nagel2020up} and BrecQ \citep{li2021brecq}, two PTQ techniques, we observe similar results as expected.
In particular, BrecQ which already achieves decent accuracy in W4/A4 with a $5.23$ points accuracy drop gets closer to the original accuracy ($0.86$ point accuracy drop) using a quarter of the expansion.
Those results demonstrate REx versatility. 

\section{Conclusion}

In this work, we proposed a novel data-free quantization method, dubbed REx, that consists in an expansion of residual quantization errors. Furthermore, we proposed a group-sparse version of the residual expansion that allows to find the best accuracy \textit{v.s.} speed trade-offs. We demonstrated the exponential convergence of the quantized weights obtained through the different expansion methods towards the full-precision model. These theoretical guarantees are crucial in the context of data-free quantization where we cannot empirically measure the accuracy degradation in an industrial application context. As such, REx allows to find superior trade-offs for several bit-width representations, which allows better flexibility and adaptability to specific hardwares. 

In particular, we showed the added value of REx through extensive empirical validation. It appears that REx significantly outperforms recent data-free quantization methods on a wide range of ConvNet architectures applied to image classification, object detection, semantic segmentation as well as transformers architectures on GLUE text classification. Furthermore, we showed that REx allows to efficiently handle outliers within the weight distributions, a well-known pitfall when attempting to quantize LLMs, using a single binary residual to account for outliers. Lastly, the ideas presented in this paper are orthogonal to most recent approaches focusing on improving the quantization operator, and hence can straightforwardly be combined with those approaches.

\subsection{Limitations:} The residual expansion method introduced in this paper does not adapt to the inter-layer importance and runtime cost discrepancies. An interesting future work would thus consist in applying more expansion orders on the most important layers w.r.t. the model accuracy, as well as using fewer orders for the most computationally expensive layers.


{\small
\bibliographystyle{unsrt}
\bibliography{output.bib}
}
\newpage
\appendix

\section{Exponential Convergence}\label{appendix:convergence}
The exponential convergence can be proved for the two methods: expansion and sparse expansion. We first prove it for the expansion on sequential models, then generalize the result to more diverse architectures.
Before detailing the proof of lemma \ref{thm:dre}, we empirically motivate the assumption of symmetry over the weight values distribution. In Figure \ref{fig:symmetry_assumption}, we plot the distributions of the weights of several layers of a ResNet 50 trained on ImageNet.
The assumption is often satisfied in practice. Furthermore, in any instances where it would not be satisfied, it can be enforced using asymmetric quantization.
\begin{lemma}\label{thm:dre}
Let $f$ be a layer with weights $W \in \mathbb{R}^n$ with a symmetric distribution. 
We denote $R^{(k)}$ the $\text{k}^{\text{th}}$ quantized weight from the corresponding residual error.
Then the error between the rescaled $W^{(K)}=Q^{-1}(R^{(K)})$ and original weights $W$ decreases exponentially, \textit{i.e.}:
\begin{equation}\label{eq:exponential_decrease}
    \left| w - \sum_{k = 1}^{K} w^{(k)} \right| \leq \left(\frac{1}{2^{b-1}-1}\right)^{K-1} \frac{{\left(s_{R^{(K)}}\right)}_i}{2}
\end{equation}
where $w$ and $w^{(k)}$ denote the elements of $W$ and $W^{(k)}$ and ${\left(s_{R^{(k)}}\right)}_i$ denotes the row-wise rescaling factor at order $k$ corresponding to $w$, as defined in equation \ref{eq:quantization_operator}.
\end{lemma}
We work on expanded layers which compute 
\begin{equation}\label{eq:define_dre}
    f^{(K)} : x \mapsto \sigma \left( \sum_{k=1}^K R^{(k)}Q(x)s_{R^{(k)}}s_x + b\right)
\end{equation}
\begin{proof}
Assume $K = 1$, then $W^{(1)}$ is the result of the composition of inverse quantization operator and quantization operator, i.e. $W^{(1)} = s_W \left\lfloor\frac{W}{s_{W}}\right\rceil$.
By definition of the rounding operator we know that $|\lfloor a \rceil - a | \leq 0.5$.
Thus we have $| w - w^{(1)} | \leq s_W / 2$.
Now in the case $k=2$, we have by definition of the quantization of the residual error and the property of the rounding operator
\begin{equation}\label{eq:6}
    \left|\left\lfloor \frac{w - w^{(1)}}{s_{R^{(2)}}} \right\rceil - \frac{w - w^{(1)}}{s_{R^{(2)}}} \right|\leq \frac{1}{2}
\end{equation}
where $s_{R^{(2)}}$ is the rescaling factor in the second order residual $R^2$ computed from $w - w^{(1)}$. The quantized weights are thus given by:
\begin{equation}
    \left| w - \sum_{i = 1}^{2} w^{(i)} \right|\leq \frac{s_{R^{(2)}}}{2}
\end{equation}
Because the weight distribution is symmetric we know that for any $k$, $s_{R^{(K)}} = \frac{\max\{w - \sum_{k = 1}^{K-1} w^{(k)}\}}{2^{b-1}-1}$ or any other definition of the delta in the full-precision space.
Also, by definition we have $\max\{w - \sum_{k = 1}^{K-1} w^{(k)}\} \leq s_{R^{(K)}}$. Thus:
\begin{equation}
    \left| w - \sum_{k = 1}^{K} w^{(k)} \right| \leq \left(\frac{1}{2^{b-1}-1}\right) \frac{s_{R^{(K)}}}{2}
\end{equation}
We conclude by using a trivial induction proof.
\end{proof}
As an immediate consequence we have the following corollary which justifies the expansion appellation:
\begin{corollary}\label{thm:expansion}
Let $f$ be a layer of real-valued weights $W$ with a symmetric distribution and $R^{(k)}$ the $\text{k}^{\text{th}}$ quantized weight from the corresponding residual error. Then,
\begin{equation}
    \mathbb{E}\left[\left\| f - \sum_{k=1}^K f^{(k)} \right\|\right] \geq \mathbb{E}\left[\left\| f - \sum_{k=1}^{K+1} f^{(k)} \right\|\right]
\end{equation}

and $f = \sum_{k=1}^\infty f^{(k)}$.

\end{corollary}
The first inequality results from detailing the induction in the previous proof.
Instead of an upper bound on the error over all the scalar values we consider each error and show using the same properties that they go down after each step. $f = \sum_{k=1}^\infty f^{(k)}$ is a direct consequence of equation \ref{eq:exponential_decrease}.

\begin{figure}[!t]
    \centering
    \includegraphics[width = 0.6\linewidth]{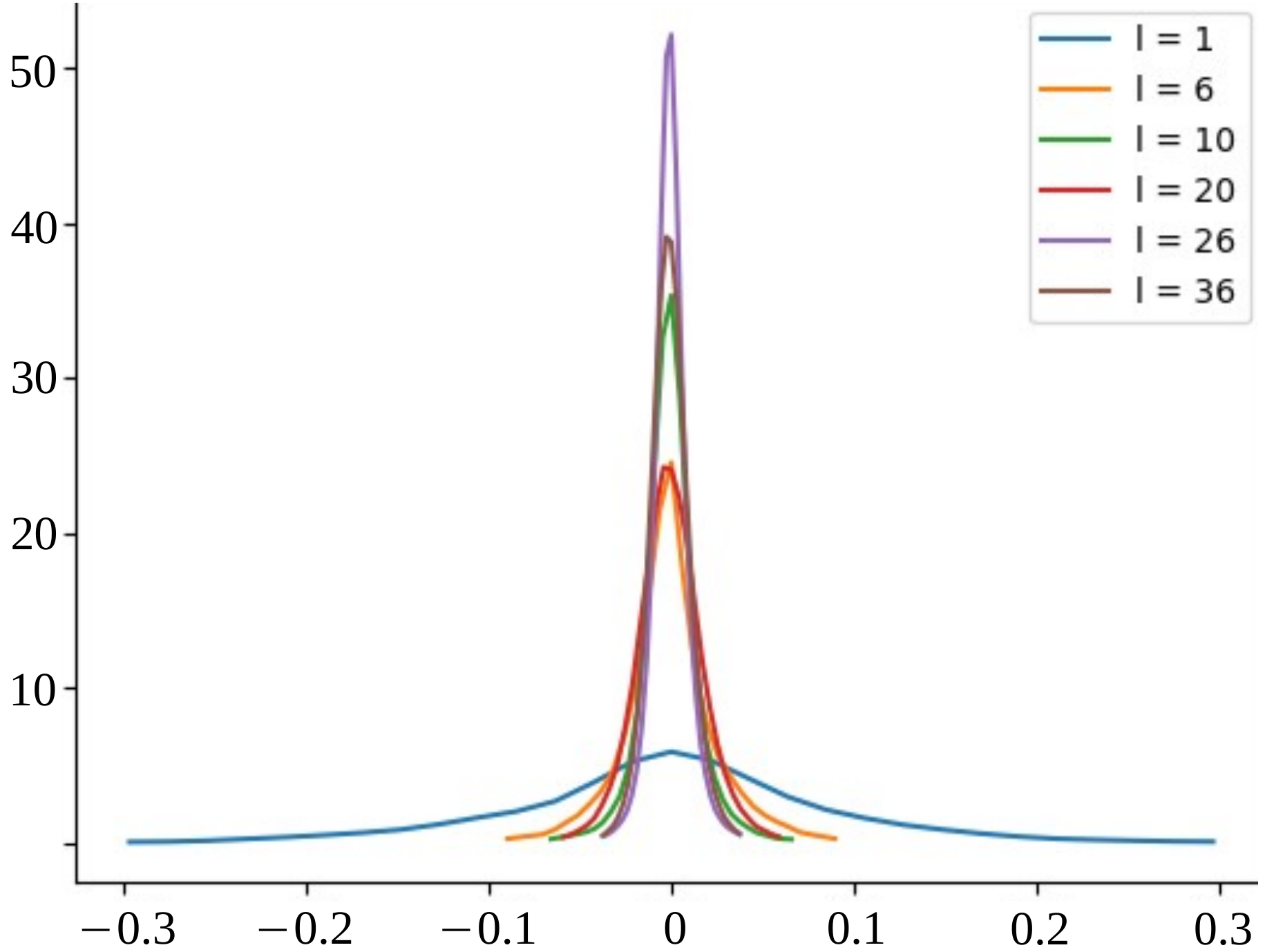}
    \caption{Distribution of the scalar weight values of different layers of a ResNet 50 trained on ImageNet. We observe that every distribution is symmetric around $0$.}
    \label{fig:symmetry_assumption}
\end{figure}

\paragraph{Sparse Expansion}
Let $N^{(k)}_i$ denotes the $L_1$ norm of an output channel $i$ of the $k$-th order residue $R^{(k)}$. The sparse residue is defined as:
\begin{equation}\label{eq:appendix_sparsity_definition}
    \left(R^{(k)}_\gamma\right)_i = (R^{(k)})_i \cdot \mathbbm{1}^{(k)}_\gamma
\end{equation}
where $\cdot$ is the element-wise multiplication, $\mathbbm{1}^{(k)}_{\gamma}=\mathbbm{1}_{\{N^{(k)}_i \geq \tau^{(k)}_{\gamma}\}}$ and $\tau^{(k)}_{\gamma}$ is a threshold defined as the ${\gamma}$ percentile of $N^{(k)}$. In other words, we remove a proportion $\gamma$\ of channels from residue $R^{(k)}$ that are the least important, as indicated by their norm $N^{(k)}$. Note however that these pruned channels can be encoded in subsequent residuals, \textit{i.e.} $R^{(k')}$, with $k' >k$. The result from Lemma \ref{thm:dre} becomes:
\begin{lemma}\label{thm:slim}
Let $f$ be a layer of real-valued weights $W$ with a symmetric distribution.
Then we have
\begin{equation}\label{eq:slim}
\begin{aligned}
    \left| w - \left(\sum_{k = 1}^{K-1} w^{(k)} + Q^{-1}\left(R^{(K)}_\gamma\right) \right) \right| \\ \leq  \frac{\left\| N^{(K)} \cdot \mathbbm{1}^{(K)}_{\gamma} \right\|_\infty{\left(s_{R^{(k)}}\right)}_i}{\left(2^{b-1}-1\right)^{K}2}
\end{aligned}
\end{equation}
where $\|\|_\infty$ is the infinite norm operator with the convention that $\|0\|_\infty = 1$ and ${\left(s_{R^{(k)}}\right)}_i$ denotes the row-wise rescaling factor at order $K$ corresponding to $w$.
\end{lemma}
\begin{proof}
From equation \ref{eq:exponential_decrease}, we have:
\begin{equation}
    \left| w - \left(\sum_{k = 1}^{K-1} w^{(k)} + Q^{-1}\left(R^{(K)}_{1}\right) \right) \right| \leq  \frac{{\left(s_{R^{(K)}}\right)}_i}{2}\left(\frac{1}{2^{b-1}-1}\right)^{K}
\end{equation}
which corresponds to the case where $\gamma^l = 1$. 
If $\gamma^l < 1$, we have two possibilities for $w$.
First, the coordinate in $N^{(K)}$ associated to is greater than $\tau^{(K)}_{\gamma^l}$ then we fall in the case where $R^{(K)}_\gamma = R^{(K)}$ and as such we have the result from equation \ref{eq:exponential_decrease} which is stronger than equation \ref{eq:slim}.
Second, the coordinate in $N^{(K)}$ associated to is lower than $\tau^{(K)}_{\gamma^l}$.
Then we have that the difference between the baseline weight $w$ and the slim expansion is bounded by the expansion of lower order and the maximum of the norm $N^{(K)}$ which leads to the result in equation \ref{eq:slim}.
\end{proof}
\paragraph{Empirical validation:} In lemma \ref{thm:dre} and \ref{thm:slim} we stated the exponential convergence to $0$ of the approximation error on the weight values.
In order to empirically confirm this theoretical result, we quantize a ResNet 50 trained on ImageNet in ternary values for different orders $K$.
As can be seen in Figure \ref{fig:norms_errors}, the average error per layer, exponentially converges to $0$ which matches our expectations.
The figure also confirms the empirical result on the strategies for $\gamma$. The higher errors are located on the last layers, thus these layers require more attention.
\begin{figure}[!t]
    \centering
    \includegraphics[width = 0.5\linewidth]{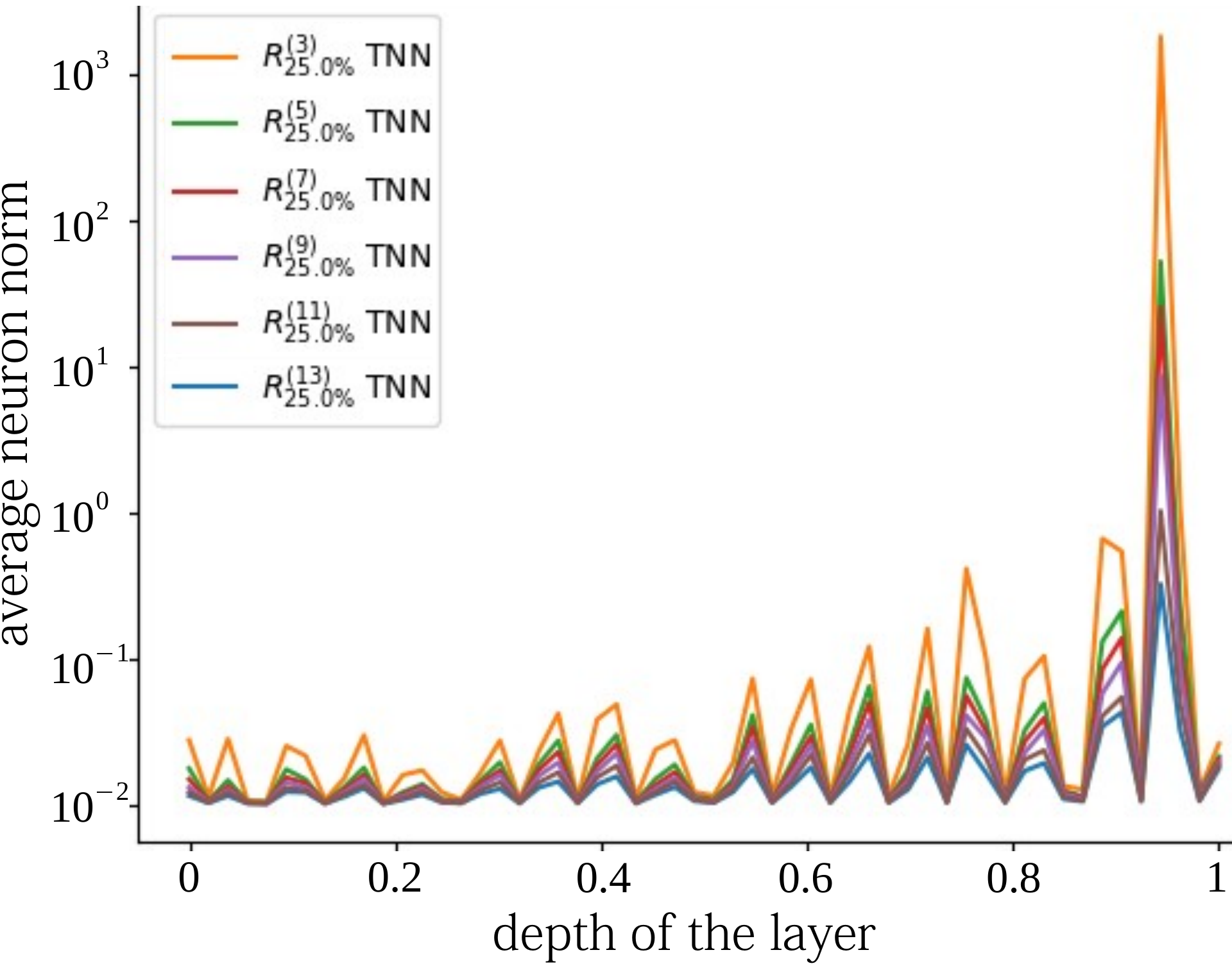}
    \caption{Comparison of the average norm of the quantization error for each layers of a ResNet 50 trained on ImageNet. We observe the exponential convergence stated in lemma \ref{thm:dre} and \ref{thm:slim}.}
    \label{fig:norms_errors}
\end{figure}

\section{Upper Bound Error}\label{appendix:upperbound}
\begin{theorem}\label{thm:dre_upperbound}
Let $F$ be a trained $L$ layers sequential DNN. We note $\sigma_l$ the largest singular value of $W_l - \sum_k R^{(k)}$, \textit{i.e.} the spectral norm of $W_l- \sum_k R^{(k)}$. Then we have
\begin{equation}
\begin{aligned}
    \max_{\|X \| = 1} \|F(X) - F(X)^{(K)}\|_\infty \leq U_{\text{res}}\\
    U_{\text{res}} = \prod_{l=1}^L \left(\sum_{i=1}^l \sigma_i u_i^{(K)} + 1 \right) - 1
\end{aligned}
\end{equation}
where $u_l^{(K)} = \left(\frac{1}{2^{b-1}-1}\right)^{K-1} \frac{{\left(s_{R^{(K)}}\right)}_i}{2}$ from equation \ref{eq:exponential_decrease}.
\end{theorem}
\begin{proof}
Let's consider $L=2$, and $F : X \mapsto B \sigma (Ax)$. For any $X$ in the domain of $F$ such that $\| X \| = 1$, we have
\begin{equation}
    \| F(X) \|_2 \leq \sigma_B + \sigma_A + \sigma_B \sigma_A
\end{equation}
where $\sigma_B$ is the largest singular value of $B$ and $\sigma_A$ is the largest singular value of $A$. Following the definition of the $2$-norm and $\infty$-norm, we get that 
\begin{equation}
    \sigma_{A - A^{(K)}} \leq \sigma_A u_A^{(K)}
\end{equation}
where $\sigma_{A - A^{(K)}}$ is the largest singular value of the residual error of order $K$, $A - A^{(K)}$ and $u_A^{(K)}$ is derived from equation \ref{eq:exponential_decrease}. Consequently, we get 
\begin{equation}
    \| F(X) - F^{(K)}(X) \|_2 \leq \sigma_B u_B^{(K)} + \sigma_A u_A^{(K)} + \sigma_B u_B^{(K)} \sigma_A u_A^{(K)}
\end{equation}
\end{proof}

\paragraph{Sparse Expansion}
\begin{theorem}\label{thm:slim_upperbound}
Let $F$ be a trained $L$ layers sequential DNN. We note $\sigma_l$ the largest singular value of $W_l- \sum_k R^{(k)}$, \textit{i.e.} the spectral norm of $W_l- \sum_k R^{(k)}$. Then we have
\begin{equation}
\begin{aligned}
    \max_{\|X \| = 1} \|F(X) - F(X)^{(K)}\|_\infty \leq U_{\text{sparse}}\\
    U_{\text{sparse}}= \prod_{l=1}^L \left(\sum_{i=1}^l \sigma_i u_i^{(K)} + 1 \right) - 1
\end{aligned}
\end{equation}
where $u_l^{(K)} = \frac{\left\| N^{(K)} \cdot \mathbbm{1}^{(K)}_{\gamma} \right\|_\infty{\left(s_{R^{(k)}}\right)}_i}{\left(2^{b-1}-1\right)^{K}2}$ from equation \ref{eq:slim}.
\end{theorem}
This results is directly derived from Theorem \ref{thm:dre_upperbound}.
This result can be extended to more sophisticated architectures. To do so we simply need to address specific attributes such as skip connections, concatenations and other activation functions.
\paragraph{Skip Connections and Concatenations}
In the case of skip connections, the graph is split from a starting layer $l_1$ and split in at least two branches that are added after layer $l_2$ and $l_3$. Assuming we can compute the upper bound for each branch (sub-networks) we simply add these sub-errors. In the case of U-nets, where skip connections contain skip connections, we simply perform this process recursively.

A similar approach can be applied to address concatenations. However in this case we keep the largest value instead of adding them.

\paragraph{Self-Attention and Cross-Attention blocks}
In order to generalize to attention modules, we need to generalize our formula to a product of layers. Let's consider the weight tensors of the keys $W_{\text{keys}}$ and queries $W_{\text{queries}}$. Then the attention scores are computed as follows
\begin{equation}\label{eq:attention}
    \text{Att}(X) = {(W_{\text{keys}} \times X)}^T \times {(W_{\text{queries}} \times X)}
\end{equation}
We want to bound the quantization error on the attention mechanism. However, the process involves the magnitude of the inputs $X$ as we highlight
\begin{equation}
    \text{Error}_\text{Att}(X)= \left \| {(W_{\text{keys}} \times X)}^T \!\times\! {(W_{\text{queries}} \times X)} - {\left(\!\!\!\left(\sum_k R_{\text{keys}}^{(k)}\right) \times X\right)}^T \!\!\!\!\times\! \left(\!\!\!\left(\sum_k R_{\text{queries}}^{(k)}\right) \times X\right) \right\|
\end{equation}
If we note $\sigma_{\text{k}}$ and $\sigma_{\text{q}}$ the spectral norms of the residual errors of the keys and queries respectively, then we can simplify the previous formulation
\begin{equation}
    \text{Error}_\text{Att}(X) = \left \| {(\sigma_{\text{k}} \times X)}^T {(W_{\text{queries}} \times X)} + {(W_{\text{keys}} \times X)}^T {(\sigma_{\text{q}} \times X)} +{(\sigma_{\text{k}} \times X)}^T {(\sigma_{\text{q}} \times X)} \right\|
\end{equation}
In order to measure this influence on the softmax in the worst case scenario, we can simply compare the $\sigma_{\text{k}}$ and $\sigma_{\text{q}}$ to the smallest singular values of $W_{\text{queries}}$ and $W_{\text{keys}}$. If we note $\alpha_k$ and $\alpha_q$ the largest singular values of $W_{\text{keys}}$ and $W_{\text{queries}}$ respectively, then we get
\begin{equation}
    \text{Error}_\text{Att}(X)\Big|_{\|X\|\leq 1} \leq \sigma_{\text{k}}\alpha_q + \sigma_{\text{q}}\alpha_k + \sigma_{\text{k}}\sigma_{\text{q}}
\end{equation}
If we note $\epsilon = \sigma_{\text{k}}\alpha_q + \sigma_{\text{q}}\alpha_k + \sigma_{\text{k}}\sigma_{\text{q}}$ this upper bound, then the error on the softmax scores becomes
\begin{equation}
    \text{Error}_\text{Softmax}(X)\Big|_{\|X\|\leq 1} \leq 1-e^{-2\epsilon}
\end{equation}

\paragraph{Other Activation Functions}
Although ReLU activations are predominant in modern DNNs, there are still many other widely used activation functions such as SiLU, GeLU or even sigmoid. SiLU and GeLU are bounded by the ReLU on the positive side which is where the highest errors occur. Consequently, the upperbound is invariant to GeLU and SiLU activation functions (although under more assumptions on the support, the upper bound could be tightened for ReLU and should be modified for GeLU and SiLU). On the other hand, for sigmoid activations or similar activations (e.g. tanh), the upper bound becomes an upper bound on $X$ in the domain of $F$ instead of $X$ on the unit circle.

\section{Sparse Expansion Outperforms Standard Expansion}\label{appendix:pyramid}
\begin{lemma}\label{thm:pyramid}
Let $f$ be a layer of real-valued weights $W$ with a symmetric distribution. Then, for $K' < K$ two integers, we have:
\begin{equation}
    \text{Err}\left( R^{(1)} + \sum_{k=2}^{K'} R^{(k)}_{\gamma_1} \right) \geq \text{Err}\left( R^{(1)} + \sum_{k=2}^{K} R^{(k)}_{\gamma_2} \right)
\end{equation}
where $\text{Err}$ is the quantization error (\textit{i.e.} the absolute difference between the quantized and original weights, as in Equation \ref{eq:exponential_decrease}) and $K' \times \gamma_1 = K \times \gamma_2 = \beta$.
\end{lemma}
\begin{proof}
Let's assume the layers outputs two channels.
Then, we have $\gamma_1 = 1$ and $\gamma_2 = 0.5$.
We simply need to prove the result for $k_1 = 2$ and $k_2 =1$ as the result will extend naturally from this case.
The idea of the proof consists in showing that using lower $\beta$ values enables more possibilities of expansions which may lead to better performance.
Let's note $(W)_1$ and $(W)_2$ the weights corresponding to the computation of the first and second output channels respectively.
Using $\gamma_1 = 1$, the second order expansion correspond to either quantizing $(W)_1$ or $(W)_2$. 
Assume $(W)_1$ is chosen for $R^{(2)}_{\gamma_1}$.
Then, $R^{(3)}_{\gamma_1}$ will either quantize the error from $(W)_2$ or further quantizes the error from $R^{(2)}_{\gamma_1}$.
In the first case we end up with $R^{(1)} + \sum_{i=2}^{k_1} R^{(i)}_{\gamma_1} = R^{(1)} + \sum_{n=2}^{k_2} R^{(i)}_{\gamma_2}$.
Otherwise, $\text{Err}\left( R^{(1)} + \sum_{i=2}^{k_1} R^{(i)}_{\gamma_1} \right) > \text{Err}\left( R^{(1)} + \sum_{i=2}^{k_2} R^{(i)}_{\gamma_2} \right)$.
\end{proof}

\section{Implementation Details and Datasets}\label{appendix:implem}\addtocontents{toc}{\newline Implementation Details and Datasets\dotfill\thepage}
We validate the proposed method on three challenging computer vision tasks which are commonly used for comparison of quantization methods.
First, we evaluate on ImageNet \citep{imagenet_cvpr09} ($\approx 1.2$M images train/50k test) classification.
Second, we report results on object detection on Pascal VOC 2012 \citep{pascal-voc-2012} ($\approx$ 17k images in the test set). Third, we benchmark on image segmentation on CityScapes dataset \citep{cordts2016cityscapes} (500 validation images). Our NLP results were obtained on the transfer learning task GLUE \citep{wang-etal-2018-glue}. We also evaluate the OPT-13B \cite{zhang2022opt} LLM on the standard common sense reasoning datasets: BoolQ \cite{clark2019boolq}, PIQA \cite{bisk2020piqa}, HellaSwag \cite{zellers2019hellaswag}, WinoGrande \cite{sakaguchi2021winogrande}, ARC easy and challenge \cite{clark2018think} and OpenBookQA \cite{mihaylov2018can}.
In our experiments we used MobileNets \citep{sandler2018MobileNetV2} and ResNets \citep{he2016deep} on ImageNet. For Pascal VOC object detection we employed an SSD \citep{liu2016ssd} architecture with MobileNet backbone.
On CityScapes we used DeepLab V3+ \citep{chen2018encoder} with MobileNet backbone.
We also test our method on VGG 16 \citep{simonyan2014very} and transformers such as BERT model \citep{devlin2018bert} as well as large language models such as OPT-13B \cite{zhang2022opt}.

In our experiments, the inputs and activations are quantized using the same method as \cite{nagel2019data}. We count the bit-wise operations as follows: let $W$ be the real-valued weights of a $d\times d$ convolutional layer on input feature maps of shape $D\times D\times n_i$ and $n_o$ outputs and stride $s$.
Then the convolutional product requires $d^2\frac{D^2}{s^2} n_i n_o$ floating point multiplications.
The quantized layer requires two rescaling operations (for the quantization of the inputs and the $Q^{-1}$ operation) and an int-$b$ convolution, i.e. $n_i D^2 + \frac{D^2}{s^2} n_o$ floating point multiplications and $d^2\frac{D^2}{s^2} n_i n_o$ int-$b$ multiplications.
Note that the number of additions remains unchanged.
According to \cite{klarreich2019multiplication} the lowest complexity for $b$-digits scalar multiplication is $o(b\log(b))$ bit operations.
This is theoretically achieved using Harvey-Hoeven algorithm (also the asymptomatic bound has yet to be proved).
We use this value as it is the least favorable setup for the proposed method.
As a consequence the number $O_{\text{original}}$ bit operations required for the original layer, $O_{R^{(1)}}$ the number of bit operations for the naively quantized layer and $O_{R^{(k)}}$ for the $\text{i}^{\text{th}}$ order residual quantization expansion are
\begin{equation}
    \begin{cases}
    O_{\text{original}} = D^2\frac{d^2n_i n_o}{s^2}  32\log(32) \\
    O_{R^{(1)}} = D^2 \left[(n_i  + \frac{n_o}{s^2}) 32\log(32) + \frac{d^2n_i n_o}{s^2} b \log(b)\right] \\
    O_{R^{(k-1)}} = D^2 \left[(n_i  + \frac{n_o}{s^2}) 32\log(32) + k \frac{d^2n_i n_o}{s^2} b \log(b)\right]
    \end{cases}
\end{equation}
Using this result we can estimate the maximum order of expansion before which the number of operations in $f^{(k)}$ exceeds the $O_{\text{baseline}}$.
Note that in the case of fully-connected layers, $D = 1$, $s = 1$ and $d = 1$.
In the following section, we use the induced metric of accuracy with respect to the total number of bit-wise operations performed by the DNN on a single input.
This metric doesn't consider the fact that the added operations can be performed in parallel. For SQuant \citep{squant2022}, we use our own implementation which achieve different accuracy results due to different initial accuracies for baseline models. As for ZeroQ \citep{cai2020zeroq}, we use results provided by SQuant \citep{squant2022}. Similarly to prior work \citep{meller2019same,nagel2019data,squant2022}, we denote W$\cdot$/A$\cdot$ the quantization setup (number of bits for weight quantization and number of bit for activation quantization).
We used Tensorflow implementations of the baseline models from the official repository when possible or other publicly available resources when necessary.
MobileNets and ResNets for ImageNet come from tensorflow models zoo.
In object detection, we tested he SSD model with a MobileNet backbone from Manish's git repository.
Finally, in image semantic segmentation, the DeepLab V3+ model came from Bonlime's git repository.
The networks pre-trained weights provide standard baseline accuracies on each tasks.
The computations of the residues as well as the work performed on the weights were done using the Numpy python's library.

\section{Expansion Reduction in the Accumulator}\label{sec:appendix_sum}
Let's go though the detailed procedure we applied in order to go from the simulated quantization with floating point scaling factors to integer only inference. We rely on the procedure introduced in [1]. First, let's consider the quantization of a single tensor $A$. Quantization is simulated using 
$A \approx s_A \times \left\lfloor A/s(A) \right\rceil = s_A \times A^Q$
In the present situation, $A^Q$ is quantized and actually fits on the target bit-width while $s_A$ is stored as a floating point value.
In order to achieve integer-only inference, we need to convert the multiplication by $s_A$ to an integer multiplication. From [1], we rely on equation (6) and, using similar notations, we get $s_A \times A^Q \approx M_A \times 2^{-n} \times A^Q$.
All these operations are integer-only operations. However, in practice, these operations may add errors on top of the quantization scheme itself. To measure this error, we conducted our own experiment and observed that the extra error does not change the quantized output of the layer; this is due to the fact that the term $M_A$ has at least 30 bits of precision while $A^Q$ has 1, 4 or 8 bits of precision (depending on the quantization bit-width). This difference in precision comes from the use of a larger accumulator which is standard in quantized inference.

Now that we detailed how to quantize, we detail how to add two distinct tensors $A$ and $B$ (which will be of special importance to add the residues, as you pointed out). In the simulated quantization, we would get
$A + B \approx s_A \times \left\lfloor A/s(A) \right\rceil + s_B \times \left\lfloor B/s(B) \right\rceil = s_A \times A^Q + s_B \times B^Q$. Now, by applying the same technique as above, we get
$s_A \times A^Q + s_B \times B^Q \approx M_A \times 2^{-n} \times A^Q + M_{B/A} \times 2^{-n} \times B^Q$, where $M_{B/A}$ is the integer closest to $M_{B}/M_{A}$ encoded with 30 bits of precision.
This was discussed in Appendix A.2 in [1]. The authors state that this operation is costly as it requires to perform the integer multiplication prior to the addition. This is a result of the fact that we need to go from the accumulator down to the quantized bit-width and then back up to the accumulator size.

In our pipeline, we limit this cost by using a fused operation in order to introduce low overhead as compared to simply using a larger kernel size. Formally, we used the above mentioned formula directly on the multiplication result. In other words, we get the following formula for $A$ and its quantized residue $R_A$ in the case of quantization using $b$ bits:
$s_A \times A^Q + s_{R_A} \times R_A \approx 2^{-n} \times (M_A \times A^Q + M_{R_A} \times 2^{-b} \times R_A)$.
Consequently, the overhead from residual summation is limited to a bit-shift on the residue during reduction of the accumulator. In Table \ref{tab:annex1}, we report the relative overhead introduced by this extra bit-shift in the residual summation scheme with respect to the total inference cost.
\begin{table}[!t]
\caption{Overhead induced by the sum reduction in full integer implementation of REx. In this table $R$ means one full residue without sparsity.}
\label{tab:annex1}
\begin{tabular}{|l|l|l|l|l|l|}
\hline
model & W$4_{+25\%}$/A4 & W$4_{+50\%}$/A4 & W$4_{+100\%}$/A4 & W$4_{+1 R}$/A4 & W$4_{+2 R}$/A4 \\ \hline
ResNet & 0.035\% & 0.070\% & 0.138\% & 0.415\% & 0.830\% \\ \hline
MobileNet v2 & 0.016\% & 0.032\% & 0.063\% & 0.189\% & 0.378\% \\ \hline
BERT & 0.004\% & 0.009\% & 0.018\% & 0.054\% & 0.107\% \\ \hline
\end{tabular}
\end{table}
For instance, we list in Table \ref{tab:annex2} some results with Gap9 hardware: we compare the overhead of computing several convolutional layers with the cost of the reduction of the residuals.
\begin{table}[!t]
\caption{Overhead induced by the sum reduction in full integer implementation of REx. In this table we study a 1 by 1 convolution on inputs of shape $224\times224\times 320$ and output of $1280$ channels as well as a depthwise convolutions on inputs of shape $224\times 224 \times 96$.}
\label{tab:annex2}
\begin{tabular}{{|l|l|l|l|l|}}
\hline
Filter & expansion order & full runtime (ops) & overhead (ops) & relative cost \\ \hline
Conv 1x1 & 2 & $2.0\times10^7$ & $6.2\times10^4$ & 0.15 \\\hline
Conv 1x1 & 3 & $4.0\times10^7$ & $6.1\times10^4$ & 0.10 \\\hline
Conv 1x1 & 4 & $6.0\times10^7$ & $6.2\times10^4$ & 0.08 \\\hline
Depthwise Conv 3x3 & 2 & $2.7\times10^5$ & $3.0\times10^4$ & 5.49 \\\hline
Depthwise Conv 3x3 & 3 & $5.4\times10^5$ & $3.0\times10^4$ & 3.38 \\\hline
Depthwise Conv 3x3 & 4 & $8.7\times10^5$ & $2.9\times10^4$ & 2.43\\\hline
\end{tabular}
\end{table}
On convolutions, as expected the cost are completely negligible. Due to the parallelization abilities of the hardware, as the expansion order increases, the overhead decreases. Furthermore, it should be noted that depthwise convolutional layers are not well supported by most hardware to this day, hence the less impressive results but similar absolute overhead cost.
\end{document}